\documentclass{article} 
\usepackage{times}
\usepackage{amsfonts}
\usepackage{amsmath}
\usepackage{algorithm}
\usepackage{algorithmic}
\usepackage{graphicx} 
\usepackage{subfigure} 
\usepackage{color}
\usepackage{authblk}

\usepackage{natbib}

\usepackage{amsthm}
\newtheorem{proposition}{Proposition}

\newif\ifenoughspace
\enoughspacefalse

\title{Bounding the Test Log-Likelihood of Generative Models}

\author[1,2]{Yoshua Bengio\thanks{yoshua.bengio@umontreal.ca}}
\author[2]{Li Yao\thanks{li.yao@umontreal.ca}}
\author[3]{Kyunghyun Cho\thanks{kyunghyun.cho@aalto.fi}}

\affil[1]{CIFAR Senior Fellow}
\affil[2]{D\'{e}partement d'Informatique et de Recherche Op\'{e}rationelle \\
Universit\'{e} de Montr\'{e}al}
\affil[3]{Department of Information and Computer Science \\ Aalto University School of Science}

\begin{document}

\maketitle

\begin{abstract}
Several interesting generative learning algorithms involve a
complex probability distribution over many random variables,
involving intractable normalization constants or latent variable
marginalization. Some of them may not have even an analytic
expression for the unnormalized probability function and no
tractable approximation. This makes it difficult to estimate the
quality of these models, once they have been trained, or to
monitor their quality (e.g.  for early stopping) while training.
A previously proposed method is based on constructing a
non-parametric density estimator of the model's probability
function from samples generated by the model. We revisit this
idea, propose a more efficient estimator, and prove that it
provides a lower bound on the true test log-likelihood and an
unbiased estimator as the number of generated samples goes to
infinity, although one that incorporates the effect of poor
mixing. We further propose a biased variant of the estimator that
can be used reliably with a finite number of samples for the
purpose of model comparison.
\end{abstract}

\section{Motivating Sampling-Based Estimators of Generative
Models' Quality}

Since researchers have 
begun considering more and more powerful
models of data distributions, they have been facing the
difficulty of estimating the quality of these models. 

In some cases, the probability distribution of a model involves
many latent variables, and it is intractable to marginalize over
those latent variables or to compute the normalization constant
(partition function). 
There exist approximation algorithms that were proposed to
overcome these intractabilities. 
One such example is Annealed Importance Sampling 
\citep[AIS,][]{Neal-2001,Salakhutdinov+Murray-2008,MurraySal09,Salakhutdinov+Hinton-2009}.
AIS, however, tends to provide optimistic estimates most of the
time, just like any estimator based on importance sampling.
This optimistic
estimation happens 
when the samples from the AIS proposal distribution
miss many important modes of the distribution.
This is
problematic because when we want to compare learning algorithms,
we 
often prefer a conservative 
estimate of
performance than an optimistic estimate that tends to over-estimate
the value of the model. This issue can be particularly troubling
when the amount of over-estimation depends on the model, 
which makes model comparisons based on an optimistic estimator dangerous. 

In other cases, one has a generative model but there is 
no explicit formulat corresponding to
the probability
function estimated by the model. That includes
Herding~\cite{WellingUAI2009}, the non-zero temperature version
of Herding~\citep{Breuleux+Bengio-2011}, and the recently
proposed generative procedures for contractive
auto-encoders~\citep[CAE,][]{Rifai-icml2012}, denoising
auto-encoders~\citep[DAE,][]{Bengio-et-al-NIPS2013}, and generative
stochastic networks~\citep[GSN,][]{Bengio+Laufer-arxiv-2013}.

For this reason, in this paper, we discuss a way to 
assess the quality of a
generative model simply by considering the samples it generates.
In the next section, we discuss the general idea of {\em
estimating a probability function} (or a density function) from
the samples generated by a generative model.
We first review a previously proposed estimator that aimed to solve
this goal. 
We show that the estimate by this estimator, in expectation over the generated
samples from a generative model, is a lower bound on the true
test log-likelihood and unbiased asymptotically.
We then propose a more efficient variant of this estimator that
has lower variance. 

\section{Previous Work}
\label{sec:previous_work}

As far as we know, \citet{Breuleux+al-TR-2010} and
\citet{Breuleux+Bengio-2011} first proposed
this kind of estimator.
The estimator computes the estimate of a geneartive model by the
following three steps:
\begin{enumerate} 
    \item Generate a set of samples $S$ from the model.
    \item Construct a non-parametric estimator $\hat{f}$ of the probability
    distribution $f$ of the model
    distribution.
    \item Compute the log-likelihood of test data under $\hat{f}$.
\end{enumerate}

In the case where the data are continuous, a Parzen density
estimator is constructed by
\begin{equation}
\hat{f}(x) = {\rm mean}_{x' \in S} {\cal N}(x;\mu=x',\sigma I),
\end{equation}
where $S$ is a set of samples from the model collected by a
Markov chain, and ${\cal N}(x;\mu,\Sigma)$ is the probability
density of $x$ under a Gaussian distribution with mean $\mu$ and
covariance $\Sigma$. In this case, the bandwidth hyper-parameter
$\sigma$ must 
be tuned according to, for instance, the log-likelihood of a
validation set.


This estimator 
was 
recently used to assess the qualities of 
the generative models such as 
stacked CAE~\citep{Rifai-icml2012-small}, 
restricted Boltzmann
machines~\citep[RBM,][]{Desjardins+al-2010-small}, deep belief
networks~\citep{Bengio-et-al-ICML2013} as well as
DAEs~\citep{Bengio-et-al-arxiv-2013v1} and
GSNs~\citep{Bengio+Laufer-arxiv-2013}.

As noted in~\citet{Breuleux+al-TR-2010}, such an estimator 
{\em measures not only the quality of the model but also that of
the generative procedure.} Any variant of this estimator will
tend to {\em estimate the log-likelihood to be lower than the
    true one when the generative procedure used to collecte
samples from a model does not mix well.} 


Another way to evaluate the quality of a generative model whose
probability is neither tractable nor easily approximated is to
use a non-parametric two-sample test~\citep{Gretton-et-al-2012}.
Unlike the approach by \citet{Breuleux+al-TR-2010}, this approach 
compares the (smoothed)
distribution of the generated samples 
to test samples 
using an $L_2$ measure (the squared error in estimated
probability).
Since in this paper we are more interested in the case of using
Kullback-Leibler divergence as a measure, we do not discuss this
approach of two-sample test any further.

\section{Conservative Sampling-based Likelihood Estimator}

We propose in this section a new estimator of the log-likelihood
of a generative model, called {\em Conservative Sampling-based
  Log-likelihood} (CSL) estimator. The proposed CSL estimator
  does not require tuning a non-parametric estimator to fit 
  samples generated from a model. It only requires that a Markov
  chain is defined for the model and used to collect samples from
  the generative model. Furthermore, we assume that the Markov
  chain alternatively samples from latent variables and observed
  variables such that the conditional distribution 
  $P(x|h)$ is well defined.

This assumption holds for many widely used generative models. An
RBM using a block Gibbs sampling is one, and a generalized
denoising autoencoder whose sampling procedure was proposed
recently by \citet{Bengio-et-al-NIPS2013-small} is another.
Multi-layered generative models such as DBNs and deep Boltzmann
machines \citep[DBM,][]{Salakhutdinov+Hinton-2009}.
  

Given the conditional probability 
$P(x|h)$ of a model and a set $S$ of samples $h'$ of the latent
variables collected from a Markov chain, the CSL estimate is
computed by
\begin{align}
\label{eq:CSL}
     \log \hat{f}(x) = \log {\rm mean}_{h' \in S} P(x|h').
\end{align}
The overall procedure of the CSL estimator is presented in
Alg.~\ref{alg:CSL}.

\begin{algorithm}[ht]
\caption{{\sc CSL}
\sl requires a set $S$ of samples of the latent variables $h'$
from a Markov chain,
a 
conditional distribution
$P(x|h)$,
and a set ${\cal X}$ of test samples.
}
\begin{algorithmic}[1]
\label{alg:CSL}
\STATE $LL=0$
\FOR{$x$ in ${\cal X}$}
\STATE $r=0$
\FOR{$h'$ in $S$}
  \STATE $r \leftarrow r + P(x|h')$
\ENDFOR
\STATE $\hat{f}_S(x)=\frac{r}{|S|}$
\STATE $LL \leftarrow LL + \log \hat{f}_S(x)$
\ENDFOR
\STATE Return 
$LL/|{\cal X}|$
\end{algorithmic}
\end{algorithm}

Unlike the original estimator described in
Sec.~\ref{sec:previous_work}, the CSL estimator utilizes the
conditional probability $P(x|h')$ rather than the actual sample
$x'$ of observed variables generated from the Markov chain. 
This has the effect of 
considerably reducing the variance of the CSL estimator. In the case of
a Gaussian conditional $P(x|h')$, whose mean $\mu$ is a function
of $h'$, for instance, 
this has the consequence of 
centering the Gaussian components of the Parzen
density estimator on the mean $\mu'$ 
rather than on the actual sample $x'$.
Since each mean $\mu'$ can ``summarize'' a very large number of
potential samples $x'$ (one could have obtained by fixing $h'$
and considering many draws from $P(x|h')$), the CSL estimator is
a much more efficient estimator with lower variance than other
estimators obtained purely from the generated samples, such as
the one described in previous section.

The other important consequence of using the conditional
distribution of the observed variables is that 
it allows us to {\em get rid of 
the bandwidth 
hyper-parameter}. Indeed, a natural choice of bandwidth (in the
case of Gaussian conditional $P(x|h')$) is precisely the standard
deviation of the Gaussian conditional distribution.  
This 
allows us to prove that the CSL
estimator is asymptotically consistent and conservative in
average later.

\section{Asymptotically Unbiased, Conservative Estimator}
\label{sec:asym}

We first prove that the CSL estimator $\log \hat{f}_S(x)$ is
asymptotically unbiased, i.e., that as the number of generated
samples increases, it approaches the ground truth log-likelihood
$\log f(x)$ associated with the stationary distribution of a
generating Markov chain.

\begin{proposition}
If the samples in $S$ are taken from chains of length $L
\rightarrow \infty$,
the CSL estimator $\log \hat{f}_S(x)$ (Algorithm~\ref{alg:CSL}, Eq.~\ref{eq:CSL})
converges to the ground truth probability $f(x)$ as the number
of samples $|S| \rightarrow \infty$, i.e., 
\begin{equation}
  \lim_{|S|\rightarrow \infty} \log \hat{f}_S(x) = \log f(x)
\end{equation}
\end{proposition}
\begin{proof}
According to the hypothesis on the samples in $S$, we have that 
Monte-Carlo estimates obtained from $S$ converge to their expectation,
i.e., the distribution of $h'$ converges to $P(h')$ under the stationary
distribution of the Markov chain.
Since $f(x)=\int P(h') P(x|h') dh'$ is the marginal distribution
over $x$ associated with this stationary distribution, its Monte-Carlo
estimator $\hat{f}(x) = {\rm mean}_{h' \ in S} P(x|h')$ converges in
the limit of $|S|\rightarrow\infty$ to its expectation under $P(h')$, i.e., $f(x)$.
Finally, note that the log of the limit equals the limit of the log.
\end{proof}

We then prove that in the finite sample case, the CSL estimator $\log \hat{f}_S(x)$
tends to underestimate the ground truth log-probability $\log f(x)$.

\begin{proposition}
The expected value of the log of the CSL estimator $\log \hat{f}_S(x)$ (Algorithm~\ref{alg:CSL})
over samples $S$ from the generative model
is a lower bound on the true log-likelihood $\log f(x)$, i.e.,
\begin{equation}
  E_S[ \log \hat{f}_S(x) ] \leq \log f(x)
\end{equation}
\end{proposition}
\begin{proof}
We simply take advantage of the concavity of the log and Jensen's inequality:
\begin{eqnarray*}
  E_S[ \log \hat{f}_S(x) ] &\leq & \log E_S[ \hat{f}_S(x) ] \\
  &=& \log E_H[ P(x|H) ] \\
  &=& \log f(x)
\end{eqnarray*}
\end{proof}

Another interesting question is the rate of convergence of the CSL estimator
to its asymptotic (ground truth) value. That rate is governed by the
variance of the estimator, which decreases linearly with the number of
samples in $S$, up to a factor which corresponds to the {\em effective
sample size} associated with the Markov chain, just like any other
Monte-Carlo average associated with the chain.

\section{Empirical Validation}

In this section, we empirically evaluate the CSL estimator on a
real dataset to investigate the rate at which the estimator
converges.

We report here the experimental result on denoising auto-encoders (DAE), generative
stochastic networks (GSN), restricted Boltzmann machines (RBM),
deep Boltzmann machines (DBM), 
and deep belief nets (DBNs). DAEs and GSNs themselves define
generative (sampling) procedures, and for RBMs, DBMs and DBNs we
used block Gibbs sampling to generate samples of latent
variables.

One interesting aspect of these experiments is that they
highlight the dependency of the estimator on the effective sample
size of the Markov chain, i.e., on the mixing rate of the
chain.  For a fixed number of samples $|S|$, chains that mix
faster provide an estimator that is 
closer to its
asymptote. In particular, these results confirm the previously
reported observations of poor mixing rate of block Gibbs sampling
for RBMs, DBMs and DBNs that are very well trained. Indeed, these
models are able to capture a sharper estimated distribution than
their less-well trained counterparts.
However, Gibbs sampling on these less-well trained models tends
to mix better~\citep{Bengio-et-al-ICML2013}, because the major
modes of the learned distribution are not separated by vast zones
of tiny probability.

All models in these experiments were trained on the binarized
MNIST data (thresholding at 0.5).  The CSL estimates of the test
set on the following models were evaluated.  For each model,
every 100-th sample from a Markov chain was collected to compute
the CSL estimate. For more details on the architecture and
training procedure of each model, see Appendix~\ref{apx:models}.

Note that although on the RBM/DBN/DBM the testset log-likelihood
is estimated by AIS (or its lower bound), there is no such AIS
estimator for DAEs and GSNs, which is where the CSL estimator may
become 
more useful.  The number of generated samples was varied
between 10,000 and 150,000.  The resulting CSL estimates are
presented in
Table~\ref{tab:CSL}.

\begin{table}[ht]
\centering
\caption{The CSL estimates 
    obtained using different numbers of 
    samples of latent variables. 
Note that samples 
are collected once every 100 sampling steps of the Markov chain.
Where available, an AIS-based estimate is also shown.}
\label{tab:CSL}
\begin{tabular}{lrrrrr}
\hline
\# samples & GSN-1  & GSN-2 &  DBN-2 & DBM-2   & RBM \\
\hline
10k        &-142    & -108  & -446   &-173   & -233 \\
50k        &-126    & -101  & -370   &-144   & -192 \\
100k       &-120    & -98   & -340   &-135   & -177 \\
150k       &-117    & -97   & -325   &-132   & -170 \\
\hline \hline
AIS        &        &       &  -57   & -76.5 & -64.1 \\
\end{tabular}
\end{table}

In the following experiment, we trained an RBM with only 5 hidden units
on MNIST,
for which 
the exact log-likelihood can be computed easily.
In this case, we observed that the CSL estimate matched the AIS
and true likelihood closely as the number of samples grew. The
CSL estimates for varying numbers of generated samples are shown
in Table~\ref{tab:exact}. 
\begin{table}[ht]
\centering
\caption{The CSL estimate converges to the true loglikelihood on
a small RBM with only 5 hidden units. }
\label{tab:exact}
\begin{tabular}{lr}
\hline
\# of Samples & Log-likelihood \\  
\hline
1k         &-188.49    \\
2k         &-186.18    \\
5k         &-182.26    \\
10k        &-181.58    \\ 
20k        &-180.65    \\ 
30k        &-180.71    \\ 
\hline \hline
exact      &-180.24    \\
AIS        &-180.22    \\      
\end{tabular}
\end{table}

\begin{figure}[t]
    \centering
    \includegraphics[width=0.8\textwidth]{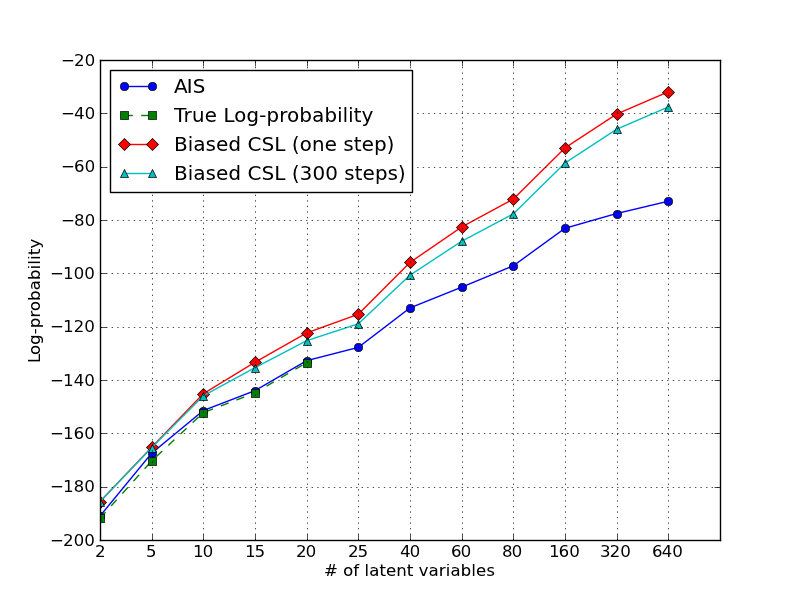}
    \label{fig:model_selection}
    \caption{The estimates of the log-probabilities of the test
        samples for 12 RBMs with varying numbers of latent
        variables. The curves represent the
        log-probabilities estimated using AIS
        (\textcolor{blue}{blue}), the biased
        CSL with a single step of 10 parallel Markov chains
        (\textcolor{red}{red}), the biased CSL
        with 300 steps of 10 parallel Markov chains
        (\textcolor{cyan}{cyan}) and
        the true log-probabilities (only for the small models,
        \textcolor{green}{green}).
    }
\end{figure}

\section{Biased CSL Estimator for Model Comparison}

Although the CSL estimator is unbiased asymptotically as shown in
Sec.~\ref{sec:asym}, it may be desirable in practice to obtain a
biased, but readily available, estimator. Hence, in this section,
we describe an algorithm, called \textit{biased CSL}, that works
with a finite number of samples. This algorithm is biased, but we
show at the end of this section, that the estimate correlate
well with the exact log-likelihood or the AIS-based estimate and
that it may be used for model comparison.

The biased CSL aims at estimating the log-probability of a
single, test sample $x$ at a time. As with the original algorithm
in Algorithm~\ref{alg:CSL}, this estimator requires only that
there are a computable conditional distribution $P(x|h)$ and a
Markov chain from where the latent variable of a model can be
sampled.

Unlike the unbiased CSL estimator, the biased CSL collected a
small set $S_x$ of \textit{consecutive} latent samples $h'$ from
a Markov chain that \textit{starts from the test sample} $x$.
This procedure ensures that the set $S_x$ will always include at
least a few samples that correspond to the neighborhood of the
test samples $x$. Furthermore, by collecting consecutive
samples, we ensure that the samples do not deviate too far away
from the starting point $x$.

Although the locality and correlatedness of the consecutive
samples $S_x$ starting from the test sample induce a bias, we
find this to be beneficial in the case of finite samples, since
the lack of any latent sample that is close to the test sample
$x$ makes the estimate highly unreliable. The biased CSL ensures
that the estimate of the probability of $x$ will be reliable and
have less variance. One consequence of the induced bias is that
the biased CSL estimator is not anymore conservative, but tends
to over-estimate the probability of the test samples.

Fig.~\ref{fig:model_selection} shows how well the biased CSL
estimates correlate with either the true log-probabilities or the
AIS-based estimates. We computed the biased CSL estimate by
running 10 parallel Gibbs sampling chains per test sample for
either a single step or 30 steps. 

It is clear from the figure that the biased CSL estimates
correlate very well with the true or AIS-based estimated
log-probabilities. As expected we see that the biased CSL
estimator tends to overestimate the log-probabilities of the test
samples. Nevertheless, we can see that the biased CSL estimator
correctly orders the model performances with only a very small
amount of samples.

This result suggests that in practice the biased CSL estimator,
which requires only a few samples per test sample, may safely be
used for the purpose of  model comparison. This is especially
useful when a model does not have an explicit probability
function, such as DAEs and GSNs.  We leave more in-depth
investigation on how the biased CSL estimator works with those
models that do not have an explicit probability function for the
future.

\section{Conclusion}

We have proposed a novel sample-based estimator for estimating the
probability that a trained model assigns to an example, called
conservative sampling-based log-likelihood (CSL) estimator. We have
justified its theoretical consistency and empirically validated
it on recently popular generative models including restricted
Boltzmann machines (RBM), deep Boltzmann machines (DBM), deep
belief networks (DBN), denoising autoencoders (DAE) and
generative stochastic networks (GSN).

The proposed CSL estimator uses only a set of samples of latent
variables generated from a model by a Markov chain. This make the
estimator useful for generative models that do not have an
explicit probability distribution but only define a generative
procedure. Also, this property of using only samples from a model
makes the estimator reflect, not only the generative performance
of the mode, but also the mixing property of the generative
procedure used to generate samples from the model. We observed
this interesting phenomenon empirically by computing the CSL
estimates on well-trained RBMs, DBNs and DBMs by generating
samples using Gibbs sampling which is known to have a poor mixing
behavior in these models.

In addition to the unbiased CSL estimator, we also proposed a
biased variant of the estimator that requires only a few
consecutive samples to approximate the probability of a single
test sample, called biased CSL estimator. The empirically
evidence suggested that the biased CSL estimator can be used to
compare models of varying complexities correctly, which makes the
CSL estimator more useful for those models without an explicit
probability function, such as GSNs, DAEs and contractive
autoencoders (CAE).

In the future, more systematic study of how the proposed CSL
estimator, both unbiased and biased, behaves with different
generative models. Especially, more empirical investigation of
applying the CSL estimator to those models without an explicit
probability distribution but only with a generative procedure
will be required.

%
%

\subsection*{Acknowledgements}
We would like to thank the developers of
Theano~\citep{bergstra+al:2010-scipy,Bastien-Theano-2012}, as well
NSERC, CIFAR, Compute Canada, and Calcul Qu\'ebec for funding.

\bibliography{strings,strings-shorter,ml,aigaion-shorter}

\begin{thebibliography}{}

\bibitem[Bastien {\em et~al.}(2012)Bastien, Lamblin, Pascanu, Bergstra,
  Goodfellow, Bergeron, Bouchard, and Bengio]{Bastien-Theano-2012}
Bastien, F., Lamblin, P., Pascanu, R., Bergstra, J., Goodfellow, I.~J.,
  Bergeron, A., Bouchard, N., and Bengio, Y. (2012).
\newblock Theano: new features and speed improvements.
\newblock Deep Learning and Unsupervised Feature Learning NIPS 2012 Workshop.

\bibitem[Bengio(2013)Bengio]{Bengio-arxiv2013}
Bengio, Y. (2013).
\newblock Estimating or propagating gradients through stochastic neurons.
\newblock Technical Report arXiv:1305.2982, Universite de Montreal.

\bibitem[Bengio {\em et~al.}(2013a)Bengio, Mesnil, Dauphin, and
  Rifai]{Bengio-et-al-ICML2013}
Bengio, Y., Mesnil, G., Dauphin, Y., and Rifai, S. (2013a).
\newblock Better mixing via deep representations.
\newblock In {\em ICML'13\/}.

\bibitem[Bengio {\em et~al.}(2013b)Bengio, Yao, Alain, and
  Vincent]{Bengio-et-al-NIPS2013}
Bengio, Y., Yao, L., Alain, G., and Vincent, P. (2013b).
\newblock Generalized denoising auto-encoders as generative models.
\newblock In {\em NIPS26\/}. Nips Foundation.

\bibitem[Bengio {\em et~al.}(2013c)Bengio, Li, Alain, and
  Vincent]{Bengio-et-al-arxiv-2013v1}
Bengio, Y., Li, Y., Alain, G., and Vincent, P. (2013c).
\newblock Generalized denoising auto-encoders as generative models.
\newblock Technical Report arXiv:1305.6663v1, Universite de Montreal.

\bibitem[Bengio {\em et~al.}(2013d)Bengio, Yao, Alain, and
  Vincent]{Bengio-et-al-NIPS2013-small}
Bengio, Y., Yao, L., Alain, G., and Vincent, P. (2013d).
\newblock Generalized denoising auto-encoders as generative models.
\newblock In {\em Advances in Neural Information Processing Systems 26
  (NIPS'13)\/}.

\bibitem[Bengio {\em et~al.}(2014)Bengio, Thibodeau-Laufer, and
  Yosinski]{Bengio+Laufer-arxiv-2013}
Bengio, Y., Thibodeau-Laufer, E., and Yosinski, J. (2014).
\newblock Deep generative stochastic networks trainable by backprop.
\newblock Technical Report arXiv:1306.1091.

\bibitem[Bergstra {\em et~al.}(2010)Bergstra, Breuleux, Bastien, Lamblin,
  Pascanu, Desjardins, Turian, Warde-Farley, and
  Bengio]{bergstra+al:2010-scipy}
Bergstra, J., Breuleux, O., Bastien, F., Lamblin, P., Pascanu, R., Desjardins,
  G., Turian, J., Warde-Farley, D., and Bengio, Y. (2010).
\newblock Theano: a {CPU} and {GPU} math expression compiler.
\newblock In {\em Proceedings of the Python for Scientific Computing Conference
  ({SciPy})\/}.
\newblock Oral Presentation.

\bibitem[Breuleux {\em et~al.}(2010)Breuleux, Bengio, and
  Vincent]{Breuleux+al-TR-2010}
Breuleux, O., Bengio, Y., and Vincent, P. (2010).
\newblock Unlearning for better mixing.
\newblock Technical Report 1349, Universit{\'{e}} de Montr{\'{e}}al/DIRO.

\bibitem[Breuleux {\em et~al.}(2011)Breuleux, Bengio, and
  Vincent]{Breuleux+Bengio-2011}
Breuleux, O., Bengio, Y., and Vincent, P. (2011).
\newblock Quickly generating representative samples from an {RBM}-derived
  process.
\newblock {\em Neural Computation\/}, {\bf 23}(8), 2053--2073.

\bibitem[Cho {\em et~al.}(2013)Cho, Raiko, and Ilin]{NECO_cho_2013_enhanced}
Cho, K., Raiko, T., and Ilin, A. (2013).
\newblock Enhanced gradient for training restricted boltzmann machines.
\newblock {\em Neural computation\/}, {\bf 25}(3), 805--831.

\bibitem[Desjardins {\em et~al.}(2010)Desjardins, Courville, Bengio, Vincent,
  and Delalleau]{Desjardins+al-2010-small}
Desjardins, G., Courville, A., Bengio, Y., Vincent, P., and Delalleau, O.
  (2010).
\newblock Tempered {{M}arkov} chain {M}onte {C}arlo for training of restricted
  {B}oltzmann machine.
\newblock In {\em AISTATS'2010\/}, volume~9, pages 145--152.

\bibitem[Gretton {\em et~al.}(2012)Gretton, Borgwardt, Rasch, Schoelkopf, and
  Smola]{Gretton-et-al-2012}
Gretton, A., Borgwardt, K., Rasch, M., Schoelkopf, B., and Smola, A. (2012).
\newblock A kernel two-sample test.
\newblock {\em J. Mach. Learning Res.}, {\bf 13}, 723--773.

\bibitem[Murray and Salakhutdinov(2009)Murray and Salakhutdinov]{MurraySal09}
Murray, I. and Salakhutdinov, R. (2009).
\newblock Evaluating probabilities under high-dimensional latent variable
  models.
\newblock In {\em NIPS'08\/}, volume~21, pages 1137--1144.

\bibitem[Neal(2001)Neal]{Neal-2001}
Neal, R.~M. (2001).
\newblock Annealed importance sampling.
\newblock {\em Statistics and Computing\/}, {\bf 11}(2), 125--139.

\bibitem[Rifai {\em et~al.}(2012a)Rifai, Bengio, Dauphin, and
  Vincent]{Rifai-icml2012}
Rifai, S., Bengio, Y., Dauphin, Y., and Vincent, P. (2012a).
\newblock A generative process for sampling contractive auto-encoders.
\newblock In {\em ICML'12\/}.

\bibitem[Rifai {\em et~al.}(2012b)Rifai, Bengio, Dauphin, and
  Vincent]{Rifai-icml2012-small}
Rifai, S., Bengio, Y., Dauphin, Y., and Vincent, P. (2012b).
\newblock A generative process for sampling contractive auto-encoders.
\newblock In {\em ICML'2012\/}.

\bibitem[Salakhutdinov and Hinton(2009)Salakhutdinov and
  Hinton]{Salakhutdinov+Hinton-2009}
Salakhutdinov, R. and Hinton, G.~E. (2009).
\newblock Deep {Boltzmann} machines.
\newblock In {\em AISTATS'2009\/}, volume~5, pages 448--455.

\bibitem[Salakhutdinov and Murray(2008)Salakhutdinov and
  Murray]{Salakhutdinov+Murray-2008}
Salakhutdinov, R. and Murray, I. (2008).
\newblock On the quantitative analysis of deep belief networks.
\newblock In W.~W. Cohen, A.~McCallum, and S.~T. Roweis, editors, {\em {ICML}
  2008\/}, volume~25, pages 872--879. ACM.

\bibitem[Welling(2009)Welling]{WellingUAI2009}
Welling, M. (2009).
\newblock Herding dynamic weights for partially observed random field models.
\newblock In {\em UAI'09\/}. Morgan Kaufmann.

\end{thebibliography}
\bibliographystyle{natbib}

\appendix
\section{Model Descriptions}
\label{apx:models}

Here we describe the architecture and training procedure of each
model:
\\
 
{\bf GSN-1 (DAE)}: 
\begin{itemize}
    \itemsep -0.5em
    \item Architecture: 784 (input) - 2000 (tanh)
    \item Noise: (input) 0.28 salt-and-pepper, (hidden) no noise 
    \item Learning: 9-step walkback \citep{Bengio-arxiv2013},
        learning rate 0.05, cross-entropy cost, 200 epochs
    \item Early-stop: visual inspection of generated samples
\end{itemize}

{\bf GSN-2}: 
\begin{itemize}
    \itemsep -0.5em
    \item Architecture: 784 (input) - 1500 (tanh) - 1500 (tanh)
    \item Noise: (input) 0.4 salt-and-pepper, (hidden 1) no
        noise, (hidden 2) white Gaussian noise with std. 2.0
    \item Learning: learning rate 0.1, cross-entropy cost, 300 epochs
    \item Early-stop: visual inspection of generated samples
\end{itemize}

{\bf DBN-2}:
\begin{itemize}
    \itemsep -0.5em
    \item Architecture: 784 (input) - 4000 (sigmoid) - 1000 (sigmoid)
    \item Learning: (1st layer) RBM from
        \citep{NECO_cho_2013_enhanced} (2nd layer) RBM with PCD-9
\end{itemize}


{\bf DBM-2}: 
\begin{itemize}
    \itemsep -0.5em
    \item Architecture: 784 (input) - 500 (sigmoid) - 1000 (sigmoid)
    \item Learning: procedure from \citep{Salakhutdinov+Hinton-2009}
\end{itemize}

{\bf RBM}:
\begin{itemize}
    \itemsep -0.5em
    \item Architecture: 784 (input) - 4000 (sigmoid) 
    \item Learning: procedure from \citep{NECO_cho_2013_enhanced}
        (enhanced gradient, adaptive learning rate and parallel
        tempering)
\end{itemize}

\end{document}